\documentclass[12pt]{article}
\usepackage{amsmath,amsfonts,amssymb}
\usepackage{graphicx,psfrag,epsf}
\usepackage{enumerate,color}
\usepackage{natbib}
\usepackage{adjustbox,booktabs,arydshln}
\usepackage{caption}
\usepackage{subcaption}

\usepackage{algorithm}
\usepackage{algorithmic}

\newcommand{\bone}{{\mathbf{1}}}
\newcommand{\bX}{{ X}}

\newcommand{\mF}{\mathcal{F}}

\newtheorem{theorem}{Theorem}[section]
\newtheorem{lemma}{Lemma}[section]

\newenvironment{proof}{\trivlist\item[\hskip \labelsep{\sc Proof:}]}
 {\unskip\nobreak\ \lower.3ex\hbox{$\Box$}\endtrivlist}

\begin{document}

  \title{Stochastic Gradient Langevin Dynamics Algorithms with Adaptive Drifts}
  
 \author{Sehwan Kim, Qifan Song, and Faming Liang
\thanks{To whom correspondence should be addressed: Faming Liang.
  F. Liang is Professor (email: fmliang@purdue.edu),
  S. Kim is Graduate Student, and Q. Song is Assistant Professor, Department of Statistics,
  Purdue University, West Lafayette, IN 47907.
 }
 }
 
\maketitle

\begin{abstract}
Bayesian deep learning offers a principled way to address many issues concerning safety of artificial intelligence (AI), such as model uncertainty,model interpretability, and prediction bias. However, due to the lack of efficient Monte Carlo algorithms for sampling from the posterior of deep neural networks (DNNs), Bayesian deep learning has not yet powered our AI system. We propose a class of adaptive stochastic gradient Markov chain Monte Carlo (SGMCMC) algorithms, where the drift function is biased to enhance escape from saddle points and the bias is adaptively adjusted according to the gradient of past samples. We establish the convergence of the proposed algorithms under mild conditions, and demonstrate via numerical examples that the proposed algorithms can significantly outperform the existing SGMCMC algorithms, such as stochastic gradient Langevin dynamics (SGLD), stochastic gradient Hamiltonian Monte Carlo (SGHMC) and preconditioned SGLD, in both simulation and optimization tasks.

\end{abstract}

\newpage

\section{Introduction}
\label{sec:intro}

During the past decade, deep learning has been the engine powering many successes of artificial 
intelligence (AI).  However, the deep neural network (DNN), as
the basic model of deep learning, still suffers from some fundamental
issues, such as model uncertainty, model interpretability, and prediction bias, 
which pose a high risk on the safety of AI.  In particular, the standard optimization 
algorithms such as stochastic gradient descent (SGD) 
produce only a point estimate for the DNN, where model uncertainty is completely ignored. 
The machine prediction/decision is blindly taken as accurate and precise, with
which the automated system might become life-threatening to humans if used in real-life settings.
The universal approximation ability of the DNN enables it to learn powerful representations that map high-dimensional features to an array of outputs. However, the representation is less interpretable,
 from which important features that govern the function of the system are hard to be identified,
 causing serious issues in human-machine trust. In addition, the DNN often contains an excessively large
 number of parameters. As a result, the training data tend to be 
 over-fitted and the prediction tends to be biased.
 
 As advocated by many researchers, see e.g. \cite{KendallGal2017} and \cite{ChaoChen2018}, Bayesian deep learning offers a principled way to address above issues. Under the Bayesian framework, a sparse 
 DNN can be learned by sampling from the posterior distribution 
 formulated with an appropriate prior distribution, see e.g. \cite{Liang2018BNN} and \cite{polson2018posterior}. For the sparse DNN, interpretability of the structure and consistency 
 of the prediction can be established under mild conditions, and 
 model uncertainty can be quantified based on the posterior samples, see e.g. \cite{Liang2018BNN}.
 However, due to the lack of efficient Monte Carlo algorithms for sampling from
 the posterior of DNNs, Bayesian deep learning has not yet powered our AI systems. 
 
 Toward the goal of efficient Bayesian deep learning, a variety of stochastic gradient Markov 
 chain Monte Carlo (SGMCMC) algorithms have been proposed in the literature, including 
 stochastic gradient Langevin dynamics (SGLD) \citep{Welling2011BayesianLV}, stochastic gradient 
 Hamiltonian Monte Carlo (SGHMC) \citep{Chen2014StochasticGH}, 
 and their variants. 
 One merit of the SGMCMC algorithms is that they are scalable,
 requiring at each iteration only the gradient on a mini-batch of data as in the SGD algorithm.
 Unfortunately, as pointed out in \cite{Dauphin2014}, DNNs often exhibit pathological curvature
 and saddle points, rendering the first-order gradient based algorithms, 
 such as SGLD, inefficient. 
 To accelerate convergence, the second-order gradient algorithms, such as stochastic gradient Riemannian Langevin dynamics (SGRLD)\citep{AhnKW2012, GirolamiG2011, PattersonTeh2013} and 
 stochastic gradient Riemannian Hamiltonian Monte Carlo (SGRHMC) \citep{Ma2015ACR}, have been developed.
 With the use of the Fisher information matrix of the target distribution, 
 these algorithms rescale the stochastic gradient noise to be isotropic near stationary points, which helps escape saddle points faster. 
 However, calculation of the Fisher information matrix can be time consuming, which makes these 
 algorithms lack scalability necessary for learning large DNNs. 
 Instead of using the exact Fisher information matrix, preconditioned SGLD (pSGLD) \citep{Li2016PreconditionedSG} approximates 
 it by a diagonal matrix adaptively updated with the current gradient information. 
 
 \cite{Ma2015ACR} provides a general framework for the existing SGMCMC algorithms (see Section \ref{recipe}), where the stochastic
 gradient of the energy function (i.e., the negative log-target distribution) 
 is restricted to be unbiased. However, this restriction is unnecessary.  As shown in the recent work, see e.g., \cite{DalalyanK2017}, \cite{SongLiang2020}, and \cite{bhatia2019bayesian}, 
 the stochastic gradient of the energy function can be biased as long as 
 its mean squared error can be upper bounded by an appropriate function of $\theta_t$, the current  sample of the stochastic gradient Markov chain. 
 On the other hand, a variety of adaptive SGD algorithms, such as 
 momentum \citep{Qian1999}, Adagrad \citep{Duchietal2011}, RMSprop \citep{TielemanH2012}, 
  and Adam \citep{KingmaB2015}, 
 have been proposed in the recent literature for dealing with the saddle point issue encountered in deep learning.  These algorithms adjust the  
 moving direction at each iteration according to the current gradient as well as the past ones. It was shown in \cite{StaibRK2019} that, 
  compared to SGD, these algorithms escape saddle points faster and can converge faster overall to the second-order stationary points. 
  
  Motivated by the two observations above, we propose a class of adaptive SGLD algorithms, 
  where a bias term is included in the drift function to enhance escape from saddle points 
  and accelerate the convergence in the presence of pathological curvatures. 
  The bias term can be adaptively adjusted based on the path of the sampler. 
  In particular, we propose to adjust the bias term based on the past gradients 
  in the flavor of adaptive SGD algorithms \citep{Ruder2016}. 
  We establish the convergence of the proposed adaptive SGLD 
  algorithms under mild conditions, and demonstrate 
  via numerical examples that the adaptive SGLD algorithms can significantly 
  outperform the existing SGMCMC algorithms, such as SGLD, SGHMC and pSGLD. 
 

 \section{A Brief Review of Existing SGMCMC Algorithms} \label{recipe}
 
 Let $\bX_N=(X_1,X_2,\ldots,X_N)$ denote a set of $N$ independent and identically distributed samples
 drawn from the distribution $f(x|\theta)$, where $N$ is the sample size and $\theta$ is the vector  of parameters. Let $p(\bX_N|\theta)=\prod_{i=1}^N f(X_i|\theta)$ denote the likelihood function, 
 let $\pi(\theta)$ denote the prior distribution of $\theta$, and let $U(\theta)=-\log p(\bX_N|\theta)-\log \pi(\theta)$ denote the energy function of the posterior distribution. If $\theta$ has a fixed 
 dimension and $U(\theta)$ is differentiable with respect to $\theta$, then the SGLD algorithm can be  used to simulate from the posterior, which iterates by
 \[
 \theta_{t+1}=\theta_t - \epsilon_{t+1} \nabla_{\theta} \tilde{U}(\theta_t) 
   + \sqrt{2\epsilon_{t+1} \tau} \eta_{t+1}, \ \ \eta_{t+1} \sim N(0,I_d),
  \]
  where $d$ is the dimension of $\theta$, $I_d$ is an $d\times d$-identity matrix, $\epsilon_{t+1}$ 
  is the learning rate, $\tau$ is the temperature, and $\nabla_{\theta} \tilde{U}(\theta)$ denotes 
  an estimate of $\nabla_{\theta} U(\theta)$ based on a mini-batch of data. 
  The learning rate can be kept as a constant or decreasing with iterations. For the former, the convergence of the algorithm was studied in \cite{Sato2014ApproximationAO} and \cite{DalalyanK2017}. 
  For the latter, the convergence of the algorithm was studied in \cite{teh2016consistency}.
  
  The SGLD algorithm has been extended in different ways. As mentioned previously, each of its existing 
  variants can be formulated as a special case of a general SGMCMC algorithm 
  given in \cite{Ma2015ACR}. Let $\xi$ denote an augmented 
  state, which may include some auxiliary components. For example, 
   SGHMC augments the state to $\xi=(\theta,v)$ by including  an auxiliary velocity component denoted by $v$.  Then the general SGMCMC algorithm is given by 
  \[
  \theta_{t+1}=\theta_t - \epsilon_{t+1} [D(\xi)+Q(\xi)] \nabla_{\xi} \tilde{H}(\xi) + 
   \Gamma(\xi) +\sqrt{2\epsilon_{t+1}\tau} Z_{t+1},  
  \]
  where $Z_{t+1} \sim N(0, D(\xi_t))$, $H(\xi)$ is the energy function of the 
  augmented system, $\nabla_{\xi} \tilde{H}(\xi)$ denotes an unbiased estimate of 
  $\nabla_{\xi} H(\xi)$, $D(\xi)$ is a positive semi-definite diffusion 
  matrix, $Q(\xi)$ is a skew-symmetric curl matrix, and
  $\Gamma_i(\xi)=\sum_{j=1}^d \frac{\partial}{\partial \xi_j}(D_{ij}(\xi)+Q_{ij}(\xi))$.
  The diffusion $D(\xi)$ and curl $Q(\xi)$ matrices can take various forms and the choice
  of the matrices will affect the rate of convergence of the sampler. 
  For example, for the SGHMC algorithm, we have $H(\xi)=U(\theta)+\frac{1}{2} v^T v$, 
  $D(\xi)=\begin{pmatrix} 0 & 0 \\ 0 & C \end{pmatrix}$ for some positive semi-definite matrix $C$, 
   and $Q(\xi)=\begin{pmatrix} 0 & -I \\ I & 0 \end{pmatrix}$. For the SGRLD algorithm, 
   we have $\xi=\theta$, $H(\xi)=U(\theta)$, $D(\xi)=G(\theta)^{-1}$, $Q(\xi)=0$, 
   where $G(\theta)$ is the Fisher information matrix of the posterior distribution. 
   By rescaling the parameter updates according to geometry information of the manifold, 
   SGRLD generally converges faster than SGLD. However, calculating the Fisher information 
   matrix and its inverse can be time consuming when the dimension of $\theta$ is high and the total sample size $N$ is large. 
   To address this issue, pSGLD approximates $G(\theta)$ using a diagonal matrix 
   and sequentially updates the approximator using the current gradient information. To be more 
   precise, it is given by 
   \[
   \begin{split} 
   G(\theta_{t+1})&={\rm diag}(1 \oslash (\lambda \bf{1}+ \sqrt{V(\theta_{t+1})})), \\
   V(\theta_{t+1})&=\beta V(\theta_{t})+(1-\beta) \nabla_{\theta} \tilde{U}(\theta_t) \odot
   \nabla_{\theta} \tilde{U}(\theta_t), 
  \end{split}
  \]
  where $\lambda$ denotes a small constant, 
  $\odot$ and $\oslash$ represent element-wise vector product and division, respectively.

 \section{Stochastic Gradient Langevin Dynamics with Adaptive Drifts} \label{ASGLDsect}
 
 Motivated by the observations that the stochastic gradient $\nabla_{\theta} \tilde{U}(\theta)$ used 
 in SGLD is not necessarily unbiased and that the past gradients can be used to 
 enhance escape from saddle points for SGD, we propose a class of adaptive SGLD algorithms, 
 where the past gradients are used to accelerate the convergence 
 of the sampler by forming a bias to the drift 
 at each iteration. A general form of the adaptive SGLD algorithm is given by 
 \begin{equation} \label{Adapteq}
 \theta_{t+1}=\theta_t - \epsilon_{t+1} (\nabla_{\theta} \tilde{U}(\theta_t) + a A_t)
   + \sqrt{2\epsilon_{t+1} \tau} \eta_{t+1},
  \end{equation}
  where $A_t$ is the adaptive bias term, $a$ is called the 
  bias factor,  and $\eta_{t+1} \sim N(0,I_d)$. 
 Two adaptive SGLD algorithms are given in what follows.
 In the first algorithm, the  bias term is constructed based on 
 the momentum algorithm \citep{Qian1999}; 
 and in the second algorithm, the bias term is constructed based on the 
  Adam algorithm \citep{KingmaB2015}.  
 
\subsection{Momentum SGLD}

It is known that SGD has trouble in navigating ravines, i.e., the regions where the energy surface curves much 
more steep in one dimension than in another, which are common around local energy minima \citep{Ruder2016,Sutton1986}. In this scenario, SGD oscillates across the slopes of the ravine 
while making hesitant progress towards the local energy minima. 
To accelerate SGD in the relevant direction and dampen oscillations, 
the momentum algorithm \citep{Qian1999} updates the moving direction at each iteration by adding 
a fraction of the moving direction of the past iteration, the so-called momentum term, 
to the current gradient. By accumulation, the momentum term increases updates 
for the dimensions whose gradients pointing in the same directions and 
reduces updates for the dimensions whose gradients change 
directions. As a result, the oscillation is reduced and the convergence is accelerated. 

\begin{algorithm}[htbp]
   \caption{MSGLD}
   \label{alg:example}
\begin{algorithmic}
   \STATE {\bfseries Input:} Data $\{x_i\}_{i=1}^N$, 
    subsample size $n$, smoothing factor $0<\beta_1<1$, bias factor $a$, temperature $\tau$, 
     and learning rate $\epsilon$;
   \STATE {\bfseries Initialization:} $\theta_0$ from an appropriate distribution, and $m_0=0$.
   \FOR{$i=1,2,\dots,$}
   \STATE Draw a mini-batch of data $\{x_j^*\}_{j=1}^n$, and calculate 
   \STATE $\theta_{t+1}=\theta_{t}-\epsilon (\nabla\tilde{U}(\theta_{t})+a m_{t})+e_{t+1}$,
   \STATE $m_{t}=\beta_1 m_{t-1}+(1-\beta_1)\nabla\tilde{U}(\theta_{t-1})$, \\
    where $e_{t+1} \sim N(0, 2 \tau \epsilon I_d)$, and $d$ is the dimension of $\theta$. 
   \ENDFOR
\end{algorithmic}
\end{algorithm}

As an analogy of the momentum algorithm in stochastic optimization, we propose the so-called momentum SGLD (MSGLD) algorithm, where
the momentum is calculated as an exponentially decaying average of 
past stochastic gradients 
and added as a bias term to the drift of SGLD. 
The resulting algorithm is depicted in Algorithm \ref{alg:example}, where
a constant learning rate $\epsilon$ is considered for simplicity. However, as mentioned in the Appendix, the algorithm also works for the case that the learning rate decays 
with iterations. The convergence of the algorithm is established in Theorem \ref{them:1}, whose proof is given in the Appendix. 

\begin{theorem}[Ergodicity of MSGLD] \label{them:1} 
Suppose the conditions (A.1)-(A.5) hold (given in Appendix),  $\beta_1 \in (0,1]$ is a constant, and the learning rate $\epsilon$ is sufficiently small. 
Then for any smooth function $\phi(\theta)$,  
\[
\frac{1}{L} \sum_{k=1}^L \phi(\theta_k)- \int_{\Theta} \phi(\theta) \pi_*(\theta) d \theta \stackrel{p}{\to} 0, \quad \mbox{as $L\to \infty$},
\]
where $\pi_*$ denotes the posterior distribution of $\theta$, and $\stackrel{p}{\to}$ denotes convergence in probability. 
\end{theorem}

Algorithm \ref{alg:example} contains a few parameters, including the subsample size $n$, smoothing factor $\beta_1$, bias factor $a$, temperature $\tau$, and learning rate $\epsilon$. Among these parameters, 
$n$, $\tau$ and $\epsilon$ are shared with SGLD and can be set 
 as in SGLD. Refer to
\cite{Nagapetyan2017} and \cite{NemethF2019} for more discussions on their settings. 
The smoothing factor $\beta_1$ 
is a constant, which is typically set to 0.9. The bias factor $a$ is also a constant, 
which is typically set to 1 or a slightly large value.

\subsection{Adam SGLD} 
 
 The Adam algorithm \citep{KingmaB2015} has been widely used in deep learning, which typically 
 converges much faster than SGD. 
 Recently, \cite{StaibRK2019} showed that Adam can be viewed as a 
 preconditioned SGD algorithm, where the preconditioner is estimated in an on-line manner and 
 it helps escape saddle points by rescaling the stochastic gradient noise to be isotropic 
 near stationary points.

  \begin{algorithm}[htbp]
   \caption{ASGLD}
   \label{alg:example2}
\begin{algorithmic}
   \STATE {\bfseries Input:} Data $\{x_i\}_{i=1}^N$, subsample size $n$, 
    smoothing factors $\beta_1$ and $\beta_2$, bias factor $a$, temperature $\tau$, and 
    learning rate $\epsilon$;
   \STATE {\bfseries Initialization:} $\theta_0$ from appropriate distribution, 
    $m_0=0$ and $V_0=0$; 
   \FOR{$i=1,2,\dots,$}
   \STATE Draw a mini-batch of data $\{x_j^*\}_{j=1}^n$, and calculate 
   \STATE $\theta_{t+1}=\theta_{t}-\epsilon(\nabla\tilde{U}(\theta_{t})+am_{t}\oslash \sqrt{V_{t}+\lambda {\bf 1}})+e_{t+1},$
   \STATE $m_{t}=\beta_1 m_{t-1}+(1-\beta_1)\nabla\tilde{U}(\theta_{t-1})$,
   \STATE $V_{t}=\beta_2 V_{t-1}+(1-\beta_2)\nabla\tilde{U}(\theta_{t-1})\odot\nabla\tilde{U}(\theta_{t-1})$, \\
   where $\lambda$ is a small constant added to avoid zero-divisors, 
   $e_{t+1} \sim N(0, 2 \tau \epsilon I_d)$, and $d$ is the dimension of $\theta$. 
   \ENDFOR
\end{algorithmic}
\end{algorithm}

 Motivated by this result, we propose the so-called Adam SGLD (ASGLD) algorithm. Ideally, we 
 would construct the adaptive bias term as follows:
 \begin{equation} \label{Adameq1}
 \begin{split}
  m_t &=\beta_1 m_{t-1} +(1-\beta_1) \nabla \tilde{U}(\theta_{t-1}), \\ 
  \tilde{V}_t &=\beta_2 \tilde{V}_{t-1}+(1-\beta_2) \tilde{U}(\theta_{t-1}) \tilde{U}(\theta_{t-1})^T, \\ 
  \tilde{A}_t &=\tilde{V}_t^{-1/2} m_t, 
  \end{split}
  \end{equation}
  where $\beta_1$ and $\beta_2$ are smoothing factors for the first and second 
  moments of stochastic gradients, respectively. 
  Since $\tilde{V}_t$ can be viewed as an 
  approximator of the true second moment matrix $E(\nabla_{\theta} \tilde{U}(\theta_{t-1}) \nabla_{\theta} \tilde{U}(\theta_{t-1})^T)$
  at iteration $t-1$, $\tilde{A}_t$ can viewed as the rescaled momentum which is isotropic near 
  stationary points. If the bias factor $a$ is chosen appropriately, ASGLD 
  is expected to converge very fast. In particular, the bias term may guide the 
  sampler to converge to a global optimal region quickly, similar 
   to  Adam in optimization. However, when the dimension of 
  $\theta$ is high, calculation of $\tilde{V}_t$ and $\tilde{V}_t^{-1/2}$ can be time consuming. 
  To accelerate computation, we propose to approximate $\tilde{V}_t$ using a diagonal matrix 
   as in pSGLD. This leads to Algorithm \ref{alg:example2}. 
   The convergence of the algorithm 
   is established in Theorem \ref{them:2}, whose proof is given in the Appendix. 
   
\begin{theorem} [Ergodicity of ASGLD] \label{them:2} 
Suppose the conditions (A.1)-(A.5) hold (given in Appendix),  $\beta_1^2<  \beta_2$ are two constants between 0 and 1, and the learning rate 
$\epsilon$ is sufficiently small.
Then for any smooth function $\phi(\theta)$,  
\[
\frac{1}{L} \sum_{k=1}^L \phi(\theta_k)- \int_{\Theta} \phi(\theta) \pi_*(\theta) d \theta \stackrel{p}{\to} 0, \quad \mbox{as $L\to \infty$},
\]
where $\pi_*$ denotes the posterior distribution of $\theta$, and $\stackrel{p}{\to}$ denotes convergence in probability. 
\end{theorem}

Compared to Algorithm \ref{alg:example}, ASGLD contains one more parameter, $\beta_2$,
which works as the smoothing factor for the second moment term and 
is suggested to take a value of 0.999 in this paper. 
   
\subsection{Other Adaptive SGLD Algorithms}

In addition to the Momentum and Adam algorithms, other optimization algorithms, such as 
AdaMax \citep{KingmaB2015} and Adadelta \citep{AdaDelta2012}, can also be incorporated into SGLD to accelerate its convergence. Other than the bias term, the past gradients can also be used to 
construct an adaptive preconditioner matrix   
in a similar way to pSGLD. Moreover, the adaptive bias and adaptive preconditioner matrix can be 
used together to accelerate the convergence of SGLD.

\section{Illustrative Examples} \label{ASGLDnum}

Before applying the adaptive SGLD algorithms to DNN models,  we first illustrate 
their performance on three low-dimensional examples. The first example is a multivariate 
Gaussian distribution with high correlation values. The second example is a multi-modal 
distribution, which mimics the scenario with multiple local energy minima.
The third example is more complicated, which mimics 
the scenario with long narrow ravines. 
 
\subsection{A Gaussian distribution with high correlation values}

Suppose that we are interested in drawing samples from $\pi(\theta)$, 
 a Gaussian distribution with the mean zero  and the covariance matrix 
 $\Sigma=\begin{pmatrix} 1 &0.9 \\ 0.9 &1 \end{pmatrix}$. 
For this example, we have 
$\nabla_{\theta} U(\theta)=\Sigma^{-1}\theta$, and set $\nabla\tilde{U}(\theta)=\nabla U(\theta)+e$ in simulations, where $\theta=(\theta_1,\theta_2)^T \in \mathbb{R}^2$ and $e\sim N(0,I_2)$.  For ASGLD, we set 
 $\tau=1$, $a=0.1$, $\epsilon=0.1$, $\beta_1=0.9$ and $\beta_2=0.999$. 
 For MSGLD, we set $\tau=1$, 
 $a=0.01$, $\beta_1=0.9$ and $\epsilon=0.1$. For comparison, SGLD was also run for this example with the same learning rate $\epsilon=0.1$. Figure \ref{correlated} shows that both ASGLD and MSGLD work well for this example, where the left panel shows that they can produce the same accurate estimate as SGLD for the covariance matrix as the number of iterations becomes large.  

\begin{figure}[htbp]
\begin{center}
\centerline{\includegraphics[height=7cm, width=\columnwidth]{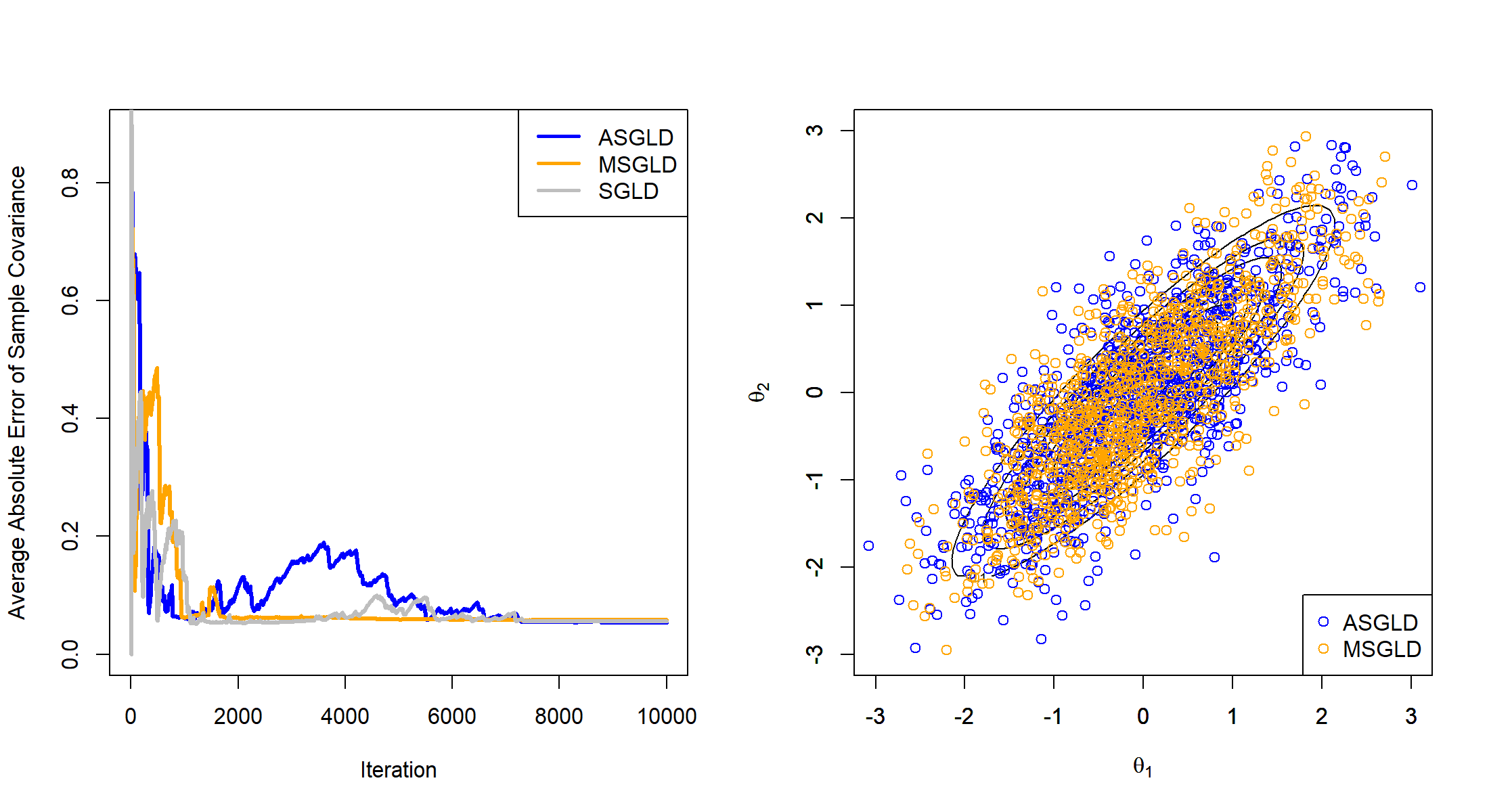}}
\caption{Performance of adaptive SGLD algorithms: (left) average absolute errors of the sample covariance matrix produced by SGLD, ASGLD and MSGLD along with iterations, and (right) scatter plots of the samples generated by ASGLD and MSGLD during their first 1000 iterations. }
\label{correlated}
\end{center}
\vskip -0.4in
\end{figure}

\subsection{A multi-modal distribution}

 The target distribution is a 2-dimensional 5-component mixture Gaussian distribution, whose density function is given by $\pi(\theta)=\sum_{i=1}^5\frac{1}{10\pi}\exp(-\|\theta-\mu_i\|^2)$, where $\mu_1=(-3,-3)^T, \mu_2=(-3,0)^T, \mu_3=(0,0)^T, \mu_4=(3,0)^T, \mu_5=(3,3)^T$. For this example, we considered the  natural gradient variational inference (NGVI) algorithm \citep{Kahn2019}, which is known to converge very fast in the variational inference field, as the baseline algorithm for comparison. 

For adaptive SGLD method, both ASGLD and MSGLD were applied to this example. We set $\nabla_{\theta} \tilde{U}(\theta)=\nabla U(\theta)+ e$, where $e\sim N(0,I_2)$ and
 $U(\theta)=-\log \pi(\theta)$. For a fair comparison, each algorithm was run in 6.0 CPU minutes. The numerical results were summarized in Figure \ref{multimodal1}, which shows the contour  of the energy function and its estimates by NGVI, MSGLD and ASGLD. The plots indicate that MSGLD and ASGLD are better at exploring the multi-modal distributions than NGVI.

\begin{figure}[htbp]
\vskip 0.2in
\begin{center}
\centerline{\includegraphics[ width=\textwidth]{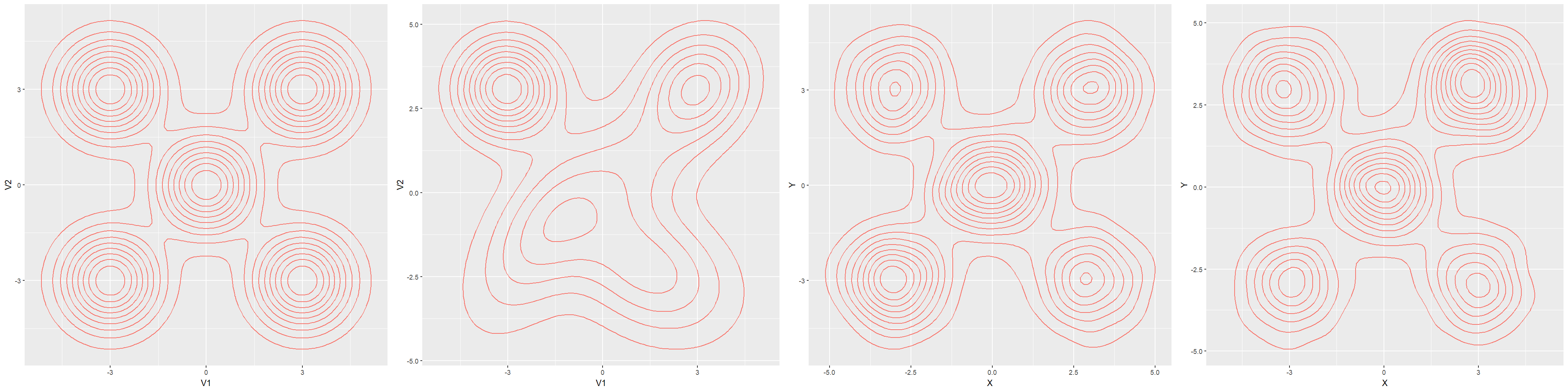}}
\caption{
Contour plots of the true energy function (left) and its estimates by NGVI (middle left), MSGLD (middle right), and ASGLD (right). 
}
\label{multimodal1}
\end{center}
\vskip -0.2in
\end{figure}

\subsection{A distribution with long narrow energy ravines}

Consider a nonlinear regression 
\[
 y=f_{\theta}(x)+\epsilon,
\quad \epsilon\sim N(0,1),
 \]
 where $x \sim Unif[-2,4]$,  $\theta=(\theta_1,\theta_2)^T \in \mathbb{R}^2$,
 and $f_{\theta}(x)=(x-1)^2+2\sin(\theta_1 x)+\frac{1}{30}\theta_1 +\cos(\theta_2 x-1)-\frac{1}{20}\theta_2$.
As $\theta$ increases, the function $f_{\theta}(x)$ fluctuates more severely. Figure 
\ref{ravine1} depicts the regression, where we set 
 $\theta_1=20$ and $\theta_2=10$. Since the random error $\epsilon$ is relatively large 
compared to the local fluctuation of $f_{\theta}(x)$, i.e., $2\sin(\theta_1 x)+\frac{1}{30}\theta_1 +\cos(\theta_2 x-1)-\frac{1}{20}\theta_2$, identification of the exact values 
of $(\theta_1, \theta_2)$ can be very hard, especially when the subsample size $n$ is small.

\begin{figure}[htbp]
    \begin{subfigure}[t]{0.48\textwidth}\caption{}
        \includegraphics[height=6cm, width=\textwidth]{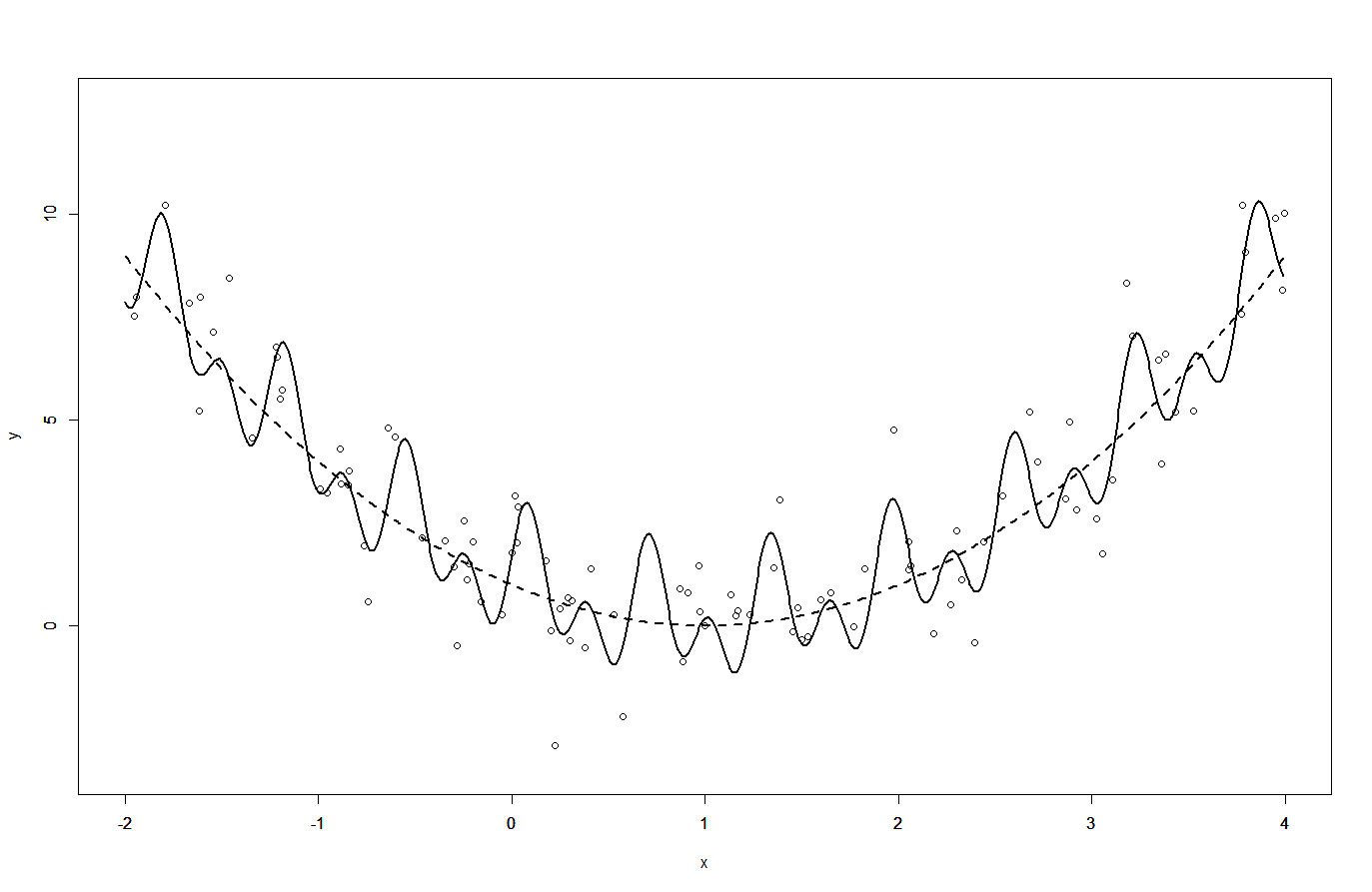}\label{ravine1}
    \end{subfigure}
    \begin{subfigure}[t]{0.48\textwidth}\caption{}
        \includegraphics[width=\textwidth]{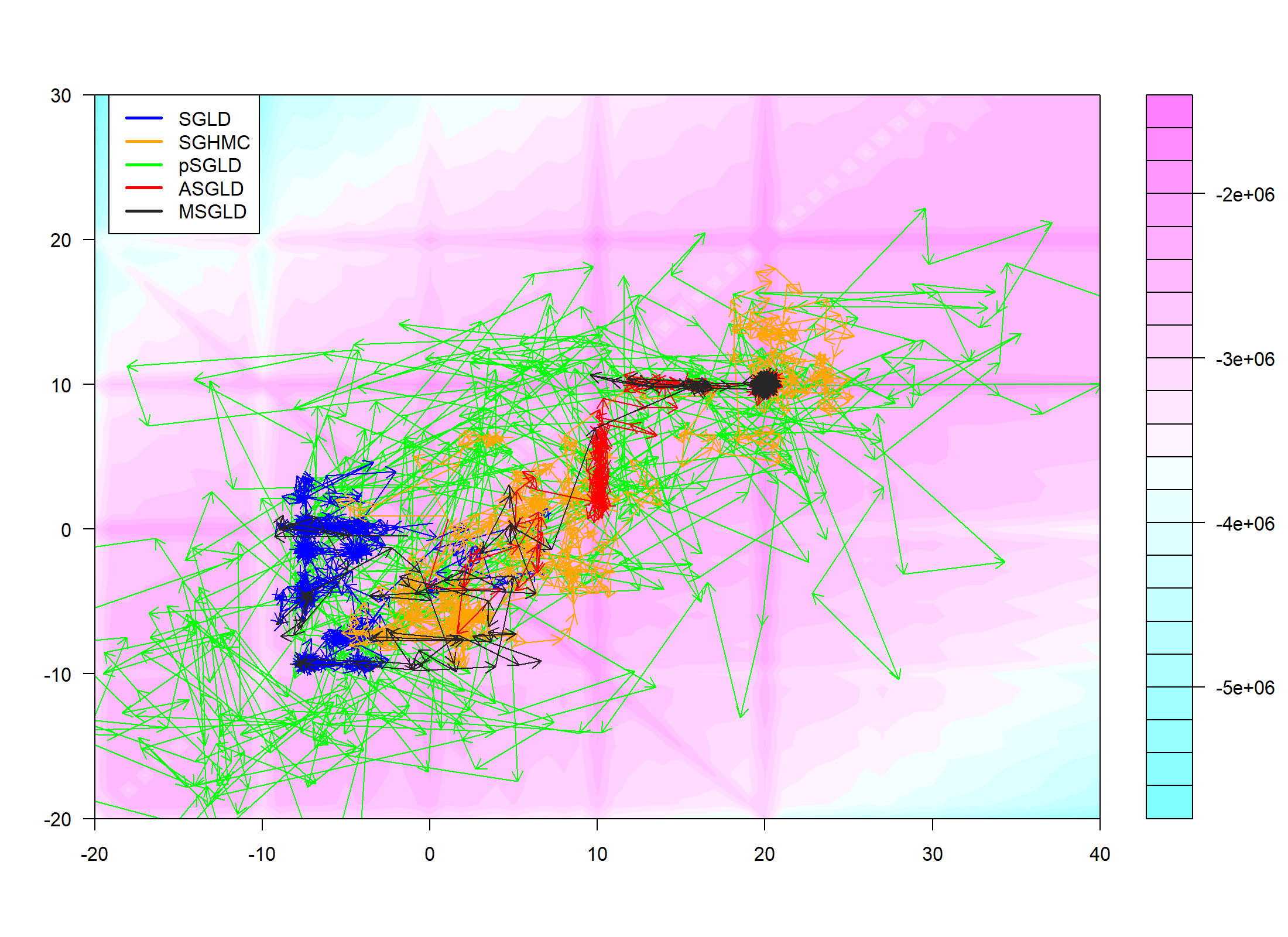}\label{ravine2}
    \end{subfigure}
  \caption{(a) The dashed line is for the global pattern $(x-1)^2$ and the solid line is for the regression function $f_{\theta}(x)$, where $\theta_1=20$ and $\theta_2=10$. The points represent 100 random samples from the regression. (b) Contour plot of the energy function for one dataset and the sample paths produced by SGLD, SGHMC, pSGLD, ASGLD and MSGLD in a run, where the sample paths have been thinned by a factor of 50 for readability of the plot.}\label{Ravine}
\end{figure}

 From this regression, we simulated 5 datasets with $(\theta_1,\theta_2)=(20,10)$ independently. 
Each dataset consists of 10,000 samples. To conduct Bayesian analysis for the problem, we 
set the prior distribution: $\theta_1 \sim N(0,1)$ and $\theta_2 \sim N(0,1)$, which are 
{\it a priori} independent. This choice of the prior distribution makes the problem even 
harder, which discourages the convergence of the posterior simulation to the true value of 
$\theta$. Instead, it encourages to estimate $f_{\theta}(x)$ by the global pattern 
$(x-1)^2$.

Both ASGLD and MSGLD were run for each of the 5 datasets. Each run consisted 
of 30,000 iterations, where the first 10,000 iterations were discarded for the burn-in process
 and the samples generated from the remaining 20,000 iterations were 
 averaged as the Bayesian estimate of $\theta$. In the simulations, we set the subsample 
 size $n=100$. The settings of other parameters are given in the Appendix. 
 Table \ref{distorted} shows the Bayesian estimates of $\theta$ produced by the two algorithms
 in each of the five runs. The MSGLD estimate converged to the true value in all five runs,
 while the ASGLD estimates converged to the true value in four of five runs. 
 For comparison, SGLD, SGHMC and pSGLD were also applied to this 
 example with the settings given in the Appendix. 
 As implied by Table \ref{distorted}, all the three algorithms essentially failed 
 for this example: none of their estimates converged to the true value!  
 
 
 \begin{table}[htbp]
\caption{Bayesian estimates of $\theta$ produced by different algorithms for the 
long narrow energy ravines example in five independent runs, 
where the true value of $\theta$ is (20,10). }
\label{distorted}
\begin{center}
\begin{sc}
\vspace{-0.1in}
\begin{tabular}{lcccccc}
\toprule
Method & $\theta$ & 1 &2 &3 & 4 &5 \\
\midrule
    & $\theta_1$ &-5.79 &-6.59 & 17.43 & 0.99 & -3.01 \\
\raisebox{1.5ex}{SGLD} &  $\theta_2$   &-2.42  &  -1.76 &9.02 & -1.36 &-4.74 \\ \hdashline
   & $\theta_1$ & 9.22 & 14.73 &23.19 & 11.39& -4.22\\
\raisebox{1.5ex}{SGHMC} &  $\theta_2$  & 1.27  & 6.23 & 11.03 & 2.65 & -2.25 \\ \hdashline
 & $\theta_1$ & 1.27 & -19.77 & 16.52 &5.14& -9.08\\
\raisebox{1.5ex}{pSGLD} &  $\theta_2$ & -2.22&  -15.44 & 8.17 & 0.65 & 3.75 \\ \midrule
 & $\theta_1$ & 19.75 & 15.34 & 19.99 & 19.01 &19.98 \\
 \raisebox{1.5ex}{ASGLD} &  $\theta_2$& 9.99 &9.29 & 9.92& 9.66 &9.99\\
\hdashline
   & $\theta_1$ &20.01 &20.07 & 19.99 & 20.00 &19.99  \\
\raisebox{1.5ex}{MSGLD} &  $\theta_2$&9.99 &10.00 &9.99 & 9.99 & 9.99\\
\bottomrule
\end{tabular}
\end{sc}
\end{center}
\vspace{-0.1in}
\end{table}
 
 For a further exploration, Figure \ref{ravine2} shows the contour plot of the energy function 
 for one dataset as well as the sample paths produced by  
 SGLD, SGHMC, pSGLD, ASGLD and MSGLD for the dataset. As shown by the contour plot, 
 the energy landscape contains multiple long narrow ravines, which make the existing 
 SGMCMC algorithms essentially fail. However, due to the use of momentum information, 
 ASGLD and MSGLD work extremely well for this example. As indicated by their  sample paths, 
 they can move along narrow ravines, and converge to the true value of 
 $\theta$ very quickly. It is interesting to point out that pSGLD does not work well for 
 this example, although it has used the past gradients in constructing the 
 preconditioned matrix. A possible reason for this failure is that it only approximates the preconditioned 
 matrix by a diagonal matrix and missed the correlation between different components of $\theta$.

\section{DNN Applications}

\subsection{Landsat Data}

 The dataset is available at the UCI Machine Learning Repository, which 
 consists of 4435 training instances and 2000 testing instances. Each instance consists of 36
 attributes which represent features of   
 the earth surface images taken from a satellite.
 The training instances consist of 6 classes, and the goal of this study is to learn
 a classifier for the earth surface images.

 We modeled this dataset by a fully connected 
 DNN with structure 36-30-30-6 and \emph{Relu} as the activation function. 
 Let $\theta$ denote the vector of all parameters (i.e., connection weights) of the DNN.  
 We let $\theta$ be subject to a Gaussian prior distribution: 
 $\theta \sim N(0,I_d)$, where $d$ is the dimension of $\theta$.  The SGLD, SGHMC, pSGLD,
 ASGLD and MSGLD algorithms were all applied to simulate from the posterior of the DNN.
 Each algorithm was run for 3,000 epochs with the subsample size 
 $n=50$ and a decaying learning rate 
 \begin{equation} \label{decayrate}
 \epsilon_{k}=\epsilon_0 \gamma^{\lfloor k/L\rfloor},
 \end{equation}
 where $k$ indexes epochs, the initial learning rate $\epsilon_0=0.1$, $\gamma=0.5$, 
 the step size $L=300$, and 
 $\lfloor z \rfloor$ denotes the maximum integer less than $z$. For the purpose of 
 optimization, the temperature was set to $\tau=0.01$. The settings for 
 the specific parameters of each algorithm were given in the Appendix. 
 
 Each algorithm was run for five times for the example. 
 In each run, the training and test classification 
 accuracy were calculated by averaging over the last $200$ samples, 
 which were collected from the last 100,000 iterations 
 with a thinning factor of $500$.  For each algorithm, 
 Table \ref{UCI-result} reports the mean classification 
 accuracy, for both training and test, averaged over five runs and its standard deviation.
 The results indicate that MSGLD has 
 significantly outperforms other algorithms in both training and test for this example. While ASGLD performs similarly to pSGLD for this example.

\begin{table}[htbp]
\caption{Training and test classification accuracy produced by different 
SGMCMC algorithms for the Landsat data, where the accuracy and its standard error 
were calculated based on $5$ independent runs.}
\label{UCI-result}
\begin{center}
\begin{sc}
\vspace{-0.1in}
\begin{tabular}{lcc} \toprule
Method & Training  Accuracy & Test Accuracy  \\
\midrule
SGLD    & 93.163$\pm$0.085  & 90.225$\pm$0.153 \\
pSGLD & 93.857$\pm$ 0.125   &  90.712$\pm$0.090\\
SGHMC & 94.015$\pm$0.117    & 90.848$\pm$0.089\\
\midrule
ASGLD & 93.827$\pm$0.087  & 90.794$\pm$0.087\\
MSGLD & {\bf 94.910} $\pm$0.105  & {\bf 91.247} $\pm$0.141 \\
\bottomrule
\end{tabular}
\end{sc}
\end{center}
\end{table}


\begin{figure}[!t]
    \begin{subfigure}[t]{0.48\textwidth}
        \caption{}\label{uci2}\includegraphics[ width=\textwidth]{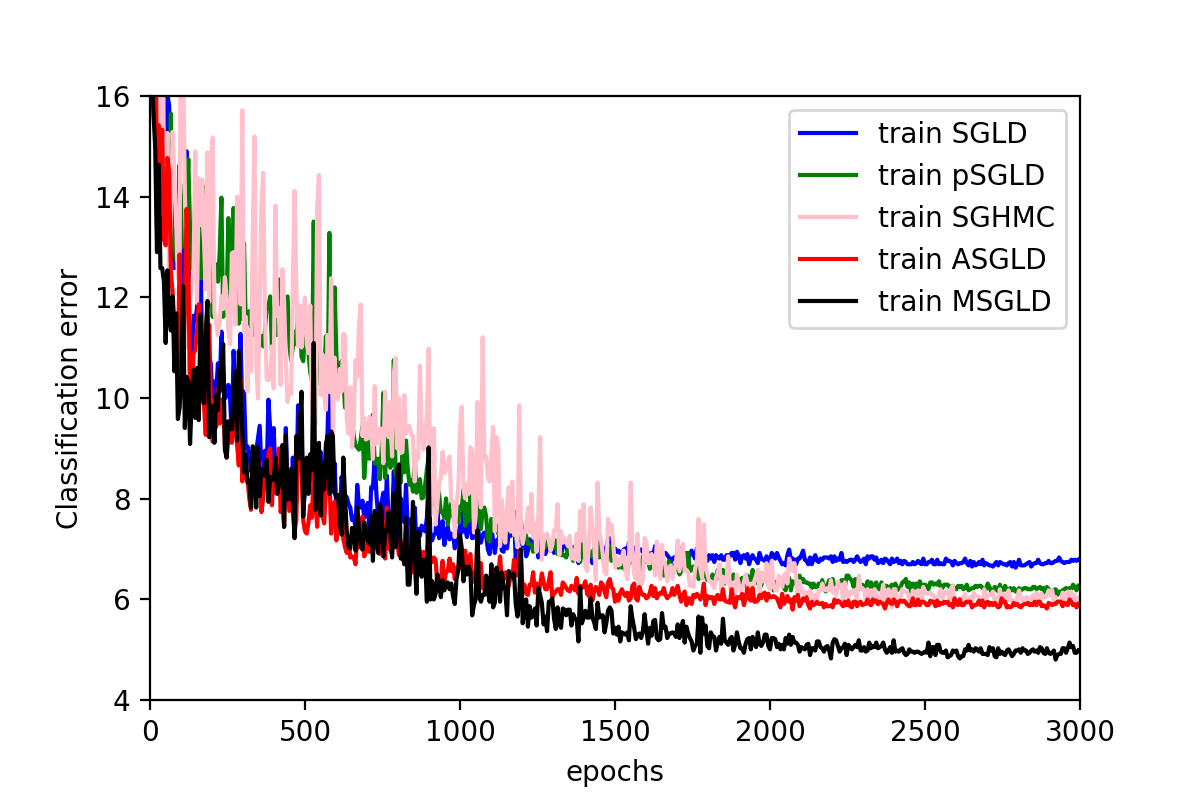}
    \end{subfigure}
    \begin{subfigure}[t]{0.48\textwidth}
        \caption{}\label{uci1}\includegraphics[width=\textwidth]{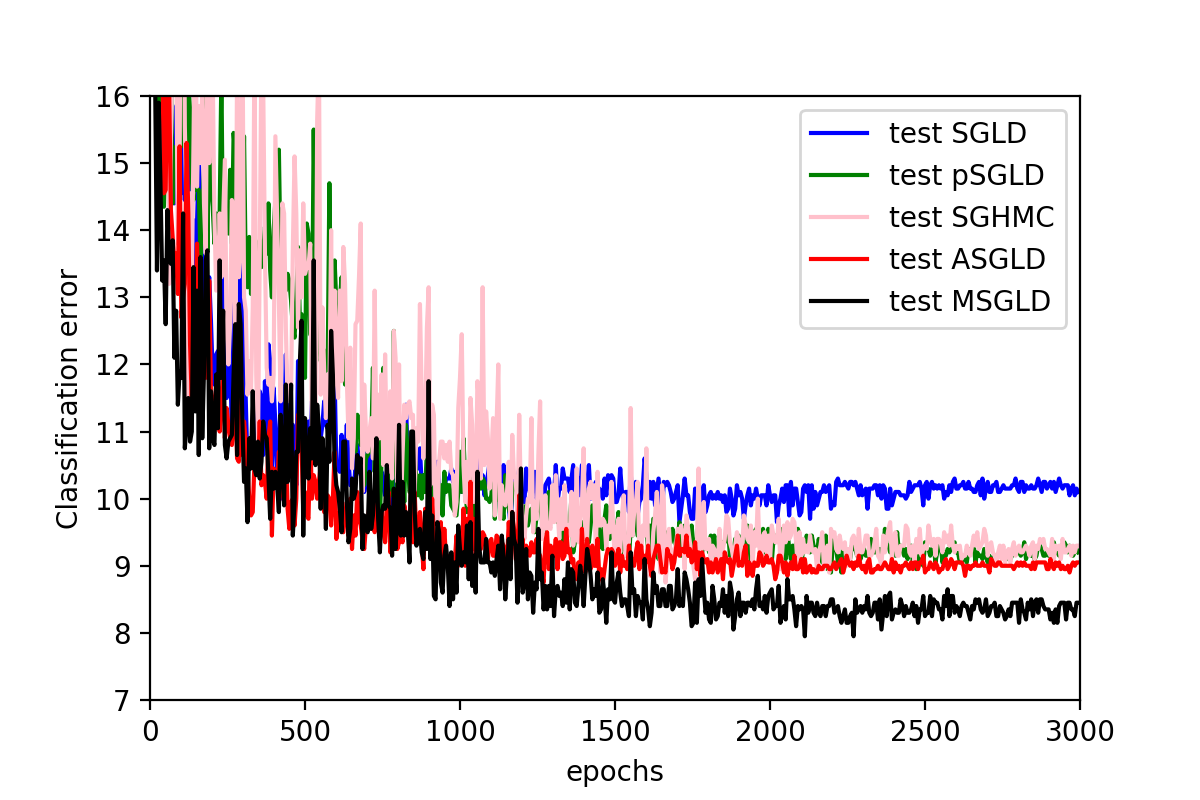}
    \end{subfigure}
    \caption{(a) Training classification errors produced by different SGMCMC algorithms 
for the Landsat data example. (b) Test classification errors produced by different SGMCMC algorithms for the Landsat data example.}\label{uci}
\end{figure}

Finally, we note that for this example, the SGMCMC algorithms have been run excessively long. 
Figure \ref{uci2} and Figure \ref{uci1} show, respectively, the
training and test classification errors produced by SGLD, pSGLD, SGHMC, 
ASGLD and MSGLD along with iterations. It indicates again that MSGLD significantly 
outperforms other algorithms in both training and test.


\subsection{MNIST data}

The MNIST is a benchmark dataset of computer vision, which consists of 60,000 training instances 
and 10,000 test instances. Each instance is an image consisting of $28\times 28$ attributes and 
representing a hand-written number of 0 to 9.   
For this data set, we tested whether ASGLD or MSGLD can be used to train sparse DNNs. 
For this purpose, we considered a mixture Gaussian prior for each of the connection weights:
\begin{equation}\label{sparseprior}
    \pi(\theta_k)\sim \lambda_k N(0,\sigma_{1,k}^2)+(1-\lambda_k)N(0,\sigma_{0,k}^2),
\end{equation}
where $k$ is the index of hidden layers, and $\sigma_{0,k}$ is a relatively very small value
compared to $\sigma_{1,k}$. In this paper, we set $\lambda_k=10^{-7},\ \sigma_{1,k}^2=0.02, \ \sigma_{0,k}^2=1\times 10^{-5}$ for all $k$. 

We trained a DNN with structure 784-800-800-10 using ASGLD and MSGLD for 250 epochs with subsample size $100$. 
For the first $100$ epochs, the DNN was trained as usual, i.e., with a Gaussian 
prior $N(0,I_d)$ imposed on each connection weight.  Then the DNN was trained 
for 150 epochs with the prior (\ref{sparseprior}). The settings for the other parameters 
of the algorithms were given in the Appendix. Each algorithm was run 
for 3 times. In each run, the training and prediction accuracy were calculated 
by averaging, respectively, the fitting and prediction results over the iterations 
of the last 5 epochs.
 The numerical results were summarized in Table \ref{sparsetab}, where ``Sparse ratio'' is calculated as the percentage of the learned connection weights satisfying the inequality
$|\theta_k|\geq \sqrt{\log(\frac{1-\lambda_k}{\lambda_k}\frac{\sigma_{1k}}{\sigma_{0k}})\frac{2\sigma_{0k}^2\sigma_{1k}^2}{\sigma_{1k}^2-\sigma_{0k}^2} }$. The threshold is determined by solving the probability inequality $P\{\theta_k \sim N(0,\sigma_{0,k}^2)|\theta_k\} \leq 
 P\{\theta_k \sim N(0,\sigma_{1,k}^2)|\theta_k\}$.

\begin{table}[htbp]
\caption{Comparison of different algorithms for training sparse DNNs for the MNIST data, where 
the reported results for each algorithm were averaged based on three independent runs. }
\label{sparsetab}
\begin{center}
\begin{sc}
\begin{tabular}{lccc}
\toprule
Method & Training Accuracy & Test Accuracy & Sparsity Ratio \\
\midrule
ASGLD &{\bf 99.542}$\pm$ 0.026 &{\bf 98.417}$\pm$0.044 & 3.176$\pm$0.155\\
MSGLD &99.494$\pm$ 0.002 &98.319 $\pm$0.019 &3.369$\pm$0.063\\
ADAM &99.383$\pm$0.072& 98.332$\pm$0.003& 2.169$\pm$0.013\\
\bottomrule
\end{tabular}
\end{sc}
\end{center}
\vskip -0.1in
\end{table}

Adam is a well known DNN optimization method for MNIST data. For comparison, it was also applied to 
train the sparse DNN. Interestingly, ASGLD outperforms Adam in both 
training and prediction accuracy, although its resulting network is a little more dense 
than that by Adam.

\subsection{CIFAR 10 and CIFAR 100}

The CIFAR-10 and CIFAR-100 are also benchmark datasets of computer vision. CIFAR-10 consists of 50,000 training images and 10,000 test images, 
and the images were classified into 10 classes. 
The CIFAR-100 dataset consists of 50,000 training images and 10,000 test images,  but 
 the images were classified into 100 classes. 
We modeled both datasets using a ResNet-18\cite{resnet}, and trained the model for 250 epochs using various optimization and SGMCMC algorithms, including SGD, Adam, SGLD, 
SGHMC, pSGLD, ASGLD and MSGLD with data augmentation techniques. 
The temperature $\tau$ was set to $1.0e-5$ for all SGMCMC methods. 
The subsample sample size was set to 100. 
For CIFAR-10, the learning rate was set as in (\ref{decayrate}) with $L=40$ and $\gamma=0.5$;  and the weight prior was set to $N(0,\frac{1}{25}I_d)$ for all SGMCMC 
algorithms. For CIFAR-100, the learning rate was set as in (\ref{decayrate}) with $L=90$ and $\gamma=0.1$; and the weight prior was set to $N(0,\frac{1}{25}I_d)$ for SGLD, pSGLD, and SGHMC, and \textcolor{black}{$N(0,\frac{1}{75}I_d)$ for ASGLD and MSGLD.}  For SGD and Adam, 
 we set the objective function as 
\begin{equation} \label{objeq3}
-\frac{1}{n} \sum_{i=1}^n \log f(x_i|\theta)+ \frac{1}{2} \lambda \|\theta\|^2,
\end{equation}
where $f(x_i|\theta)$ denotes the likelihood function of observation $i$, and $\lambda$ is the 
regularization parameter. For SGD we set $\lambda=5.0e-4$; and for Adam we set $\lambda=0$. For Adam, we have also tried the case $\lambda\ne 0$, but the results 
were inferior to those reported below.

For each dataset, each algorithm was run for 5 times. In each run, the training and test classification 
errors were calculated by averaging over the iterations of the last 5 epochs.   
Table \ref{cifar-result} reported the mean training and test classification accuracy (averaged over 
5 runs) of the Bayesian ResNet-18 for CIFAR-10 and CIFAR-100. 
The comparison indicates that 
MSGLD outperforms all other algorithms in test accuracy for the two datasets. In terms of training accuracy,
ASGLD and MSGLD work more like the existing momentum-based algorithms such as ADAM and SGHMC, but tend to generalize better than them. 
This result has been very remarkable,
as all algorithms were run with a small number of epochs; that is, 
the Monte Carlo algorithms are not necessarily slower than the 
stochastic optimization algorithms in deep learning. 
 
\begin{table}[htbp]
\caption{Mean training and test classification accuracy (averaged over 5 runs) of the Bayesian ResNet-18 
for CIFAR-10 and CIFAR-100 data}
\label{cifar-result}
\begin{center}
\begin{sc}
\begin{tabular}{lccccc}
\toprule
 & \multicolumn{2}{c}{CIFAR-10} &  & 
  \multicolumn{2}{c}{CIFAR-100}  \\  \cline{2-3} \cline{5-6}  
 Method  & Training & Test & &  Training & Test \\ \midrule 
SGD   & 99.721$\pm$0.058  &92.896$\pm$0.141 & &   
  98.211$\pm$0.208 & 72.274$\pm$0.145\\
ADAM  & 99.515$\pm$0.087 &  92.144$\pm$0.154& &  
96.744$\pm$0.390 & 67.564$\pm$0.501\\
\midrule
SGLD  & 99.713$\pm$0.034 &  92.908$\pm$0.114  & &   
97.650$\pm$0.131 & 71.908$\pm$0.149\\
pSGLD & 97.209$\pm$0.465 & 89.398$\pm$0.379& & 
 99.034$\pm$0.101 & 69.234$\pm$0.162\\
SGHMC &  99.152$\pm$0.060 & 93.470$\pm$0.073 & &  
96.92$\pm$0.091 & 73.112$\pm$0.139 \\
\midrule
ASGLD &  99.738$\pm$0.052 &  93.018$\pm$0.187 & &  
96.938$\pm$0.268 &  73.252$\pm$0.681\\
MSGLD & 99.702$\pm$0.066 & {\bf 93.512}$\pm$0.081 & & 
96.801$\pm$0.150 & {\bf 73.670}$\pm$0.144\\ \bottomrule
\end{tabular}
\end{sc}
\end{center}
\end{table}

\section{Conclusion} \label{ASGLDconc}

This paper has proposed a class of adaptive SGMCMC algorithms by including a bias term to the
drift of SGLD, where the bias term is allowed to be adaptively adjusted 
with past samples, past gradients, etc. The proposed algorithms have extended the framework 
of the existing SGMCMC algorithms to an adaptive one
in a flavor of extending the 
Metropolis algorithm \cite{Metropolis1953} to the adaptive Metropolis algorithm \citep{Haario2001} in the history of MCMC. 
The numerical results indicate that the proposed algorithms have 
inherited many attractive properties, such as quick convergence in the scenarios with   
long narrow ravines or saddle points, 
from their counterpart optimization algorithms, 
while ensuring more extensive exploration of the sample space than 
the optimization 
algorithms due to their simulation nature. As a result, the proposed algorithms can 
significantly perform the existing optimization and SGMCMC algorithms in both simulation 
and optimization. 

For the adaptive SGMCMC algorithms, different bias terms represent different strengths 
in escaping from local traps, saddle points, or long narrow ravines. In the future, we 
will consider to develop an adaptive SGMCMC algorithm with a complex bias term which has 
incorporated all the strengths in escaping from local traps, saddle points, long narrow 
ravines, etc. We will also consider to incorporate other advanced SGMCMC techniques 
into adaptive SGMCMC algorithms to further improve their performance. 
For example, cyclical SGMCMC \citep{csgmcmc} proposed a cyclical step size schedule, 
where larger steps  discover new modes,  and smaller steps characterize each mode. 
This technique can be easily incorporated into adaptive SGLD algorithms to improve 
their convergence to the stationary distribution.

\appendix

\section*{\bf Appendix}

\section{Proofs of Theorems \ref{them:1} and \ref{them:2}}

Considered a generalized SGLD algorithm with a biased drift term for simulating from the target distribution $\pi_*(\theta) \propto \exp\{-U(\theta)\}$.  Let $\theta_{k+1}$ and $\theta_k$ be two random vectors in $\Theta$ satisfying 
\begin{equation}\label{algeq1}
 \theta_{k+1} = \theta_{k}-\epsilon [\nabla U(\theta_{k})+\zeta_{k+1}]+\sqrt{2\epsilon}e_{k+1},
\end{equation}
where $e_{k+1} \sim N(0,I_p)$, and $\zeta_{k+1}=\nabla \widehat{U}(\theta_k)-\nabla U(\theta_k)$ denotes deviation between the drift 
$\nabla \widehat{U}(\theta_k)$ used in simulations and the ideal drift $\nabla U(\theta_k)=-\nabla \log\pi_*(\theta_k)$.  
For example, in equation (\ref{Adapteq}), we have 
$\nabla \widehat{U}(\theta_k)=\nabla_{\theta} \widetilde{U}(\theta_k) + a A_k$. 

 
  
 For the generalized SGLD algorithm (\ref{algeq1}), we aim to analyze the deviation of the averaging estimate $\hat{\phi}=\frac{1}{L}\sum_{k=1}^L \phi(\theta_k)$
  from the posterior mean $\bar{\phi}=\int_{\Theta} \phi(\theta) \pi_*(d\theta)$ for a bounded smooth function $\phi(\theta)$ of interest.
The key tool we employed in the analysis is the Poisson equation 
which is used to characterize the fluctuation 
between $\phi$ and $\bar \phi$: 
\begin{equation} \label{poissoneqf}
    \mathcal{L}g(\theta)=\phi(\theta)-\bar \phi,
\end{equation}
where $g(\theta)$ is the solution to the Poisson equation, and $\mathcal{L}$ is the infinitesimal generator of the Langevin diffusion 
\begin{equation*}
    \mathcal{L}g:=\langle\nabla g, \nabla U(\cdot)\rangle+\tau\Delta g.
\end{equation*}
By imposing the following regularity conditions on the function $g(\theta)$, we can control the fluctuation of $\frac{1}{L}\sum_{k=1}^L \phi(\theta_k)-\bar \phi$, which enables convergence of the sample average.

\begin{itemize}
\item[(A.1)] Given a sufficiently smooth function $g(\theta)$ as defined in (\ref{poissoneqf}) and a function $\mathcal{V}(\theta)$ such that the derivatives satisfy the inequality $\|D^j g\|\lesssim \mathcal{V}^{p_j}(\theta)$ for some constant
 $p_j>0$, where $j\in\{0,1,2,3\}$. In addition, $\mathcal{V}^p$ has a bounded expectation, i.e., $\sup_{k} E[\mathcal{V}^p(\theta_k)]<\infty$;
 and $\mathcal{V}^p$ is smooth, i.e. $\sup_{s\in (0, 1)} \mathcal{V}^p(s\theta+(1-s)\vartheta)\lesssim \mathcal{V}^p(\theta)+\mathcal{V}^p(\vartheta)$ for all $\theta,\vartheta\in\Theta$ and $p\leq 2\max_j\{p_j\}$.
\end{itemize}

For a stronger but verifiable version of the condition, we refer readers to \cite{VollmerZW2016}. In what follows, we present a lemma which is adapted from Theorem 3 of \cite{ChenNL15} with a fixed learning rate $\epsilon$.
\textcolor{black}{Note that \cite{ChenNL15} requires $\{\zeta_k:k=1,2,\ldots\}$ to be a zero mean sequence, while in our case  $\{\zeta_k: k=1,2,\ldots\}$ forms an auto-regressive sequence which makes the proof of \cite{ChenNL15} still go through.}
A similar lemma can also be established for a decaying learning rate sequence. 
Refer to Theorem 5 of \cite{ChenNL15} for the detail.  

\begin{lemma} \label{lemmanew1} 
Assume the condition (A.1) hold. 
For a smooth function $\phi$, the mean square error (MSE) of the generalized SGLD algorithm (\ref{algeq1}) at time $T = \epsilon L$ is bounded as 
\begin{equation} \label{MSEeq}
E\|\hat{\phi}-\bar{\phi}\|^2 \lesssim  
\frac{1}{L^2} \sum_{k=1}^L E\|\zeta_k\|^2 +\frac{1}{L\epsilon}+\epsilon^2. 
\end{equation}
\end{lemma}

To prove Theorems \ref{them:1} and \ref{them:2}, we further make the following assumptions: 
 
 \begin{itemize} 
 \item[(A.2)] (smoothness) $U(\theta)$ is $M$-smooth; that is, there exists a constant $M>0$ such that for any $\theta, \theta'\in \Theta$,
\begin{equation} \label{ass_2_1_eq}
\|\nabla  U(\theta)-\nabla U(\theta')\|  \leq M\|\theta-\theta'\|. \\
\end{equation}
\end{itemize}

The smoothness of $\nabla U(\theta)$ is a standard assumption in studying the convergence of SGLD, and it has been used in a few work, see e.g.  \citet{Maxim17} and \citet{Xu18}.

\begin{itemize}
    \item[(A.3)] (Dissipativity) There exist constants $m>0$ and $b\geq 0$ such that for any $\theta \in \Theta$, 
\begin{equation} 
\label{eq:01}
\langle \nabla U(\theta), \theta\rangle\geq m \|\theta\|^2 -b.
\end{equation}
\end{itemize}

This assumption has been widely used in proving the geometric ergodicity of dynamical systems \citep{mattingly02, Maxim17, Xu18}. It ensures the sampler to move towards the origin regardless the position of the current point.

\begin{itemize} 
\item[(A.4)] (Gradient noise)  The stochastic gradient $\xi(\theta)=\nabla \widetilde{U}(\theta)-\nabla U(\theta)$ is unbiased; that is, for any $\theta \in \Theta$,
$E[\xi(\theta)]=0$.
In addition, there exists some constant $B>0$ such that the second moment of the stochastic gradient is bounded by $E \|\xi(\theta)\|^2 \leq M^2 \|\theta\|^2+B^2$, 
where the expectation is taken with respect  to the distribution of the gradient noise. 
\end{itemize}

\begin{lemma} (Uniform $L^2$ bound) \label{lemmanew2}
Assume the conditions (A.2)-(A.4) hold. 
For any $0 < \epsilon < Re(\frac{m-\sqrt{m^2-4M^2(M^2+1)}}{4M^2(M^2+1)})$,  there exists a constant $G>0$ such that 
$E\|\theta_k\|^2 \leq G$, where 
$G=\|\theta_0\|^2+\frac{1}{m}(b+2\epsilon B^2(M^2+1)+p)$, $p$ denotes the dimension of $\theta$,
and $Re(\cdot)$ denotes the real part of a complex number.
\end{lemma}
\begin{proof} The proof follows that of Lemma 1 in
\cite{DengL2019}. To make use of that proof, we can rewrite equation (\ref{algeq1}) as 
\[
\theta_{k+1} = \theta_{k}-\epsilon \nabla_{\theta} L(\theta_{k},\zeta_{k+1})+\sqrt{2\epsilon}e_{k+1},
\]
where $\nabla_{\theta} L(\theta_k,\zeta_{k+1})=[\nabla_{\theta} U(\theta_k)+\zeta_{k+1}]$ by viewing $\zeta_{k+1}$ as an argument of the function $L(\cdot,\cdot)$. Then it is easy to verify that the conditions (A.2)-(A.4) imply the conditions of Lemma 1 of \cite{DengL2019}, and thus the uniform $L^2$ bound holds. \textcolor{black}{Note that given the condition (A.3), the inequality $E\langle \nabla_{\theta} L(\theta_k,\zeta_{k+1}),\theta_k\rangle \geq m E\|\theta_k\|^2-b$, required by \cite{DengL2019} in its proof, will hold as long as $\epsilon$ is sufficiently small or $\beta_1$ is not very close to 1.}
\end{proof}



 Let $\theta_*$ denote the minimizer of $U(\theta)$. Therefore, $\nabla U(\theta_*)=0$. Then, by Lemma \ref{lemmanew2} and condition (A.2), there exists a constant $C_1$ such that
\begin{equation}\label{pathbound}
   E\|\nabla U(\theta_{k})\|^2 \leq 2M^2 (G+\|\theta_*\|^2) :=C_1 <\infty.
\end{equation}

 Let $\xi_{k+1}:=\nabla\tilde U(\theta_{k})-\nabla U(\theta_{k})$ be the gradient estimation error.
 By Lemma \ref{lemmanew2} and condition (A.4), 
 there exists a constant $C_2$ such that 
\begin{equation} \label{pathbound2}
     E(\|\xi_k\|^2) \leq  
     M^2 G+B^2 := C_2< \infty.
\end{equation}

\subsection{Proof of Theorem 3.1}


\begin{proof}
The update of the MSGLD algorithm can be rewritten as 
$\theta_{k+1}=\theta_{k}-\epsilon[\nabla U(\theta_{k})+\zeta_{k+1}]+\sqrt{2\epsilon}e_{k+1}$, where 
$\zeta_{k+1}=am_{k}+\xi_{k+1}$ and $m_0=0$.

First, we study the bias of $\zeta_{k}$.
According to the recursive update rule of $m_i$, we have 
\[
\begin{split}
    &E(\zeta_{k+1}|\mathcal F_k)/a=E(m_k|\mF_k)
    = (1-\beta_1)\nabla U(\theta_{k-1})+    \beta_1 E(m_{k-1}|\mF_k)\\
    =&(1-\beta_1)\nabla U(\theta_{k-1})+ (1-\beta_1)\beta_1\nabla U(\theta_{k-2})
      +\beta_1^2 E(m_{k-2}|\mF_k) =\cdots \\
    =&\sum_{i=1}^{k} (1-\beta_1)\beta_1^{i-1}\nabla U(\theta_{k-i})+\beta_1^{k} E(m_{0}|\mF_k) 
    =\sum_{i=1}^{k} (1-\beta_1)\beta_1^{i-1}\nabla U(\theta_{k-i}). \\
\end{split}
\]
Hence, by Jensen's inequality
\[
\begin{split}
    &\|E(\zeta_{k+1}|\mathcal F_k)\|\leq a\sum_{i=1}^{k} (1-\beta_1)\beta_1^{i-1}\|\nabla U(\theta_{k-i})\| 
     \leq a\sqrt{\sum_{i=1}^{k} (1-\beta_1)\beta_1^{i-1}\|\nabla U(\theta_{k-i})\|^2}. \\
\end{split}
\]
By (\ref{pathbound}), the bias is further bounded by 
 \begin{equation} \label{biasbound}
 \begin{split}
    E\|E(\zeta_{k+1}|\mathcal F_k)\|^2 & \leq a^2{\sum_{i=1}^{k} (1-\beta_1)\beta_1^{i-1}E\|\nabla U(\theta_{k-i})\|^2}  
     \leq a^2 C_1.
\end{split}
\end{equation}

For the variance of $\zeta_{k+1}$, we have 
\[
\begin{split}
    &E\|\zeta_{k+1} - E(\zeta_{k+1}|\mF_k)\|^2=E\|\xi_{k+1}+am_k-E(am_k|\mF_k)\|^2\\
    =&E\|\xi_{k+1}+a(1-\beta_1)\tilde \nabla U(\theta_{k-1})+    a\beta_1 m_{k-1} 
      -a(1-\beta_1)\nabla U(\theta_{k-1})-  a\beta_1 E(m_{k-1}|\mF_k)\|^2\\
    =&E\|\xi_{k+1}+ a(1-\beta_1)\xi_{k} +   a\beta_1 m_{k-1}- a \beta_1 E(m_{k-1}|\mF_k)\|^2 = \cdots \\
    =&E\|\xi_{k+1}+ \sum_{i=1}^k a(1-\beta_1)\beta_1^{i-1}\xi_{k-i+1} \|^2.
\end{split}
\]
Due to the independence among $\xi_k$'s, we have 
\begin{equation}\label{vbound}
\begin{split}
    E\|\zeta_{k+1} - E(\zeta_{k+1}|\mF_k)\|^2
    &\leq E\|\xi_{k+1}\|^2  
     + \sum_{i=1}^k a^2(1-\beta_1)^2\beta_1^{2i-2}E\|\xi_{k-i+1}\|^2\\
     &\leq  C_2 [1+a^2 (1-\beta_1)/(1+\beta_1)],
    \end{split}
\end{equation}
where the last inequality follows from (\ref{pathbound2}). 

Combining (\ref{biasbound}) and (\ref{vbound}), 
we have 
\[
E\|\zeta_{k+1}\|^2 \leq a^2 C_1+  C_2[1+a^2 (1-\beta_1)/(1+\beta_1)] < \infty,
\]
which conclude the proof by applying 
Lemma \ref{lemmanew1} and Chebyshev's inequality. 
\end{proof}

\subsection{Proof of Theorem 3.2}

\begin{proof}
The update of the ASGLD algorithm can be rewritten as 
$\theta_{k+1}
=\theta_{k}-\epsilon[\nabla U(\theta_{k})+\zeta_{k+1}]+\sqrt{2\epsilon}e_{k+1}$, where 
$\zeta_{k+1}=am_{k}\oslash\sqrt{v_{k}+\lambda \bone}+\xi_{k+1}$.

According to the recursive update rule of $m_i$ and $v_i$, we have
\[
\begin{split}
    m_i =& (1-\beta_1)\tilde U(\theta_{i-1})+(1-\beta_1)\beta_1\tilde U(\theta_{i-2}) 
     +(1-\beta_1)\beta_1^2\tilde U(\theta_{i-3})+\cdots\\
    v_i =& (1-\beta_2)\tilde U(\theta_{i-1})\odot \tilde U(\theta_{i-1})+(1-\beta_2)\beta_2\tilde U(\theta_{i-2}) 
     \odot\tilde U(\theta_{i-2})
    +(1-\beta_2)\beta_2^2\tilde U(\theta_{i-3})\odot\tilde U(\theta_{i-3})+\cdots
\end{split}
\]
Therefore, by Cauchy-Schwarz inequality, when $\beta_1^2< \beta_2$ , we have
\[
\begin{split}
& \|m_{i-1}\oslash \sqrt{v_{i-1}}\|_\infty\leq \sqrt{\sum_{j=1}^{i-1} \frac{(1-\beta_1)^2\beta_1^{2j-2}}{(1-\beta_2)\beta_2^{j-1}}} 
\leq \sqrt{\frac{(1-\beta_1)^2}{1-\beta_2}\frac{1}{1-\beta_1^2/\beta_2}}:=C. \\
\end{split}
\]
It implies that 
$\|m_{i-1}\oslash\sqrt{v_{i-1}+\lambda \bone}\|\leq \sqrt{p}C$ almost surely, and in consequence, 
 \begin{equation} \label{biasbound2}
    E\|E(\zeta_{k+1}|\mathcal F_k)\|^2\leq a^2C^2p,
\end{equation}
and, by (\ref{pathbound2}),
\begin{equation}\label{vbound2}
\begin{split}
   E\|\zeta_{k+1} - E(\zeta_{k+1}|\mF_k)\|^2 &\leq E\|\zeta_{k+1}\|^2 
     \leq E\|\xi_{k+1}\|^2+a^2E\|m_{k}\oslash\sqrt{v_{k}+\lambda \bone}\|^2 \\
    & \leq a^2C^2p+C_2.
\end{split}
\end{equation}

Combining (\ref{biasbound2}) and (\ref{vbound2}), we have
\[
E\|\zeta_{k+1}\|^2 \leq 2 a^2C^2p+C_2<\infty,
\]
which concludes the proof by applying Lemma \ref{lemmanew1} and Chebyshev's inequality. 
\end{proof}

\section{Experimental Setup}


All numerical experiments on deep learning were done with pytorch. For all SGMCMC algorithms, 
the initial learning 
rates were set at the order of $O(1/N)$ in all experiments except for in MNIST training. 
For the optimization methods such as SGD and Adam, the objective function was set to  (\ref{objeq3}),
where $f(x_i|\theta)$ denotes the likelihood function of observation $i$, and $\lambda$ is the 
regularization parameter whose value varies for different datasets. 

\paragraph{Multi-modal distribution}

 For NGVI method, we chose a 5-component mixture Gaussian distribution, and set the initial parameters as $\pi_i=\frac{1}{5}$, $\mu_i\sim N(0_2,1.5I_2)$, $\Sigma_i=I_2$ for $i=1,\dots,5$, and the learning rate $\epsilon=0.005$. For MSGLD, we set $(a,\beta_1)=(10, 0.9)$ and the learning rate $\epsilon=0.05$. For Adam SGLD, we set $(a, \beta_1,\beta_2)=(1, 0.9,0.999)$ and the learning rate $\epsilon=0.05$. The CPU time limit was set to $6$ minutes.

\paragraph{Distribution with long narrow ravines} 
Each algorithm was run for $30,000$ iterations with the settings of specific parameters given in Table \ref{ravinetab}. 

\begin{table}[htbp]
\caption{Parameter setting for the distribution with long narrow ravines}
\label{ravinetab}
\begin{center}
\begin{tabular}{lccccc}
\toprule
Method & Initial value & $\beta_1$ & $\beta_2$ & \text{a} & $\lambda$ \\
\midrule
SGLD & $1e-4$& & & \\
SGHMC & $1e-5$& 0.9& &\\
pSGLD & $1e-4$& 0.9& &&$1e-6$\\
ASGLD & $1e-4$& 0.9&0.999 &1000& $1e-5$\\
MSGLD & $1e-4$& 0.99 & &10&\\
\bottomrule
\end{tabular}
\end{center}
\end{table}

\paragraph{Landsat}
Each algorithm was run for 3000 epochs with the settings of specific parameters given in 
Table \ref{UCItab}.

\begin{table}[htbp]
\caption{Parameter setting for the Landsat data example}
\label{UCItab}
\begin{center}
\begin{tabular}{lcccccc}
\toprule
Method & Initial value & $\beta_1$ & $\beta_2$ & \text{a} & $\lambda$  \\
\midrule
SGLD & $0.1/4435$& & & &\\
SGHMC &$0.1/4435$& 0.9& &&\\
pSGLD & $0.1/4435$& 0.9& &&$1e-5$\\
ASGLD & $0.1/4435$& 0.9&0.999 &10& $1e-5$\\
MSGLD & $0.1/4435$& 0.9 & &5&\\
\bottomrule
\end{tabular}
\end{center}
\end{table}

\paragraph{MNIST} 

Each algorithm was run for 250 epochs, where the first 100 epochs were run 
with the conventional Gaussian prior $N(0,1)$, and the followed 150 epochs were run with the mixture Gaussian prior given in the paper.
For Adam, the objective function was set as (\ref{objeq3}) 
 with $\lambda=0$ for all 250 epochs. 
The settings of the specific parameters were given in Table 
\ref{MNISTtab1}.

\begin{table}[htbp]
\caption{Parameter settings for MNIST before (stage I) and after (stage II) sparse learning}
\label{MNISTtab1}
\begin{center}
\begin{adjustbox}{width=1.0\textwidth}
\begin{tabular}{lcccccccccccccc} \toprule
 & \multicolumn{6}{c}{Stage I} & & \multicolumn{6}{c}{Stage II} \\ \cline{2-7} \cline{9-14}
Method & Initial & $\beta_1$ & $\beta_2$ & \text{a} & $\lambda$  &$\tau$ & & Initial & $\beta_1$ & $\beta_2$ & \text{a} & $\lambda$ & $\tau$ \\
\midrule
ADAM & $0.001$&0.9 &0.999 & &$1e-8$& & & $0.0001$&0.9 &0.999 & &$1e-8$ &  \\
ASGLD & $0.5$& 0.9&0.999 &10& $1e-6$&$1e-2$ & & $0.5/10$& 0.9&0.999 &1& $1e-8$& $1e-5$ \\
MSGLD & $0.5$& 0.99 & &1&&$1e-3$ & & $0.5/10$& 0.99 & &1&&$1e-5$ \\ \bottomrule
\end{tabular}
\end{adjustbox}
\end{center}
\end{table}

\paragraph{Cifar-10 and Cifar-100}

For SGD, the objective function was set as (\ref{objeq3}) with $\lambda=5.0e-4$. 
For Adam, the objective function was set as (\ref{objeq3}) with $\lambda=0$. The settings of specific parameters are given in Table \ref{cifar10-100}.

\begin{table}[htbp]
\caption{Parameter settings for CIFAR-10 and CIFAR-100}
\label{cifar10-100}
\begin{center}
\begin{adjustbox}{width=1.0\textwidth}
\begin{tabular}{lccccccccccc} \toprule
 & \multicolumn{5}{c}{CIFAR-10}& & \multicolumn{5}{c}{CIFAR-100} \\ \cline{2-6} \cline{8-12}
Method & Initial & $\beta_1$ & $\beta_2$ & \text{a} & $\lambda$ & & Initial & $\beta_1$ & $\beta_2$ & \text{a} & $\lambda$ \\ \midrule
SGD & $0.1$& & & &   &  & $0.1$& & & & \\
ADAM & $0.001$& 0.9& 0.999& &$1e-8$ & & $0.001$& 0.9& 0.999& &$1e-8$ \\
SGLD & $0.1/50000$& & & & & & $0.1/50000$& & & & \\
SGHMC & $0.1/50000$& 0.9& && & &$0.1/50000$& 0.9& &&  \\
pSGLD & $0.001/50000$& 0.99& &&$1e-6$ & & $0.0001/50000$& 0.99& &&$1e-6$ \\
ASGLD & $0.1/50000$& 0.9&0.999 & 20 & $1e-6$& & $0.1/50000$& 0.9&0.999 & 20 & $1e-8$\\
MSGLD &$0.1/50000$& 0.9 & &1& &&$0.1/50000$& 0.9 & &1& \\ \bottomrule
\end{tabular}
\end{adjustbox}
\end{center}
\end{table}

\bibliographystyle{asa}
\bibliography{JCGS2c}

\begin{thebibliography}{38}
\newcommand{\enquote}[1]{``#1''}
\expandafter\ifx\csname natexlab\endcsname\relax\def\natexlab#1{#1}\fi

\bibitem[{Ahn et~al.(2012)Ahn, Korattikara, and Welling}]{AhnKW2012}
Ahn, S., Korattikara, A., and Welling, M. (2012), \enquote{Bayesian Posterior
  Sampling via Stochastic Gradient Fisher Scoring,} in \textit{ICML}.

\bibitem[{Bhatia et~al.(2019)Bhatia, Ma, Dragan, Bartlett, and
  Jordan}]{bhatia2019bayesian}
Bhatia, K., Ma, Y.-A., Dragan, A.~D., Bartlett, P.~L., and Jordan, M.~I.
  (2019), \enquote{Bayesian Robustness: A Nonasymptotic Viewpoint,}
  \textit{arXiv preprint arXiv:1907.11826}.

\bibitem[{Chen(2018)}]{ChaoChen2018}
Chen, C. (2018), \enquote{Uncertainty Estimation of Deep Neural Networks,}
  \textit{PhD dissertation, University of South Carolina, U.S.A.}

\bibitem[{Chen et~al.(2015)Chen, Ding, and Carin}]{ChenNL15}
Chen, C., Ding, N., and Carin, L. (2015), \enquote{On the Convergence of
  Stochastic Gradient {MCMC} Algorithms with High-order Integrators,} in
  \textit{NeurIPS}, pp. 2278--2286.

\bibitem[{Chen et~al.(2014)Chen, Fox, and Guestrin}]{Chen2014StochasticGH}
Chen, T., Fox, E.~B., and Guestrin, C. (2014), \enquote{Stochastic Gradient
  {H}amiltonian {M}onte {C}arlo,} in \textit{ICML}.

\bibitem[{Dalalyan and Karagulyan(2017)}]{DalalyanK2017}
Dalalyan, A.~S. and Karagulyan, A.~G. (2017), \enquote{User-friendly guarantees
  for the {L}angevin {M}onte {C}arlo with inaccurate gradient,} \textit{CoRR},
  abs/1710.00095.

\bibitem[{Dauphin et~al.(2014)Dauphin, Pascanu, Gulcehre, Cho, Ganguli, and
  Bengio}]{Dauphin2014}
Dauphin, Y.~N., Pascanu, R., Gulcehre, C., Cho, K., Ganguli, S., and Bengio, Y.
  (2014), \enquote{Identifying and attacking the saddle point problem in
  high-dimensional non-convex optimization,} in \textit{Advances in Neural
  Information Processing Systems 27}, eds. Ghahramani, Z., Welling, M., Cortes,
  C., Lawrence, N.~D., and Weinberger, K.~Q., Curran Associates, Inc., pp.
  2933--2941.

\bibitem[{Deng et~al.(2019)Deng, Zhang, Liang, and Lin}]{DengL2019}
Deng, W., Zhang, X., Liang, F., and Lin, G. (2019), \enquote{An {A}daptive
  {E}mpirical {B}ayesian {M}ethod for {S}parse {D}eep {L}earning,} in
  \textit{NeurIPS}.

\bibitem[{Duchi et~al.(2011)Duchi, Hazan, and Singer}]{Duchietal2011}
Duchi, J., Hazan, E., and Singer, Y. (2011), \enquote{Adaptive subgradient
  methods for online learning and stochastic optimization,} \textit{\JMLR}, 12,
  2121--2159.

\bibitem[{Girolami and Calderhead(2011)}]{GirolamiG2011}
Girolami, M. and Calderhead, B. (2011), \enquote{Riemann manifold Langevin and
  Hamiltonian Monte Carlo methods (with discussion),} \textit{\JRSSB}, 73,
  123--214.

\bibitem[{Haaro et~al.(2001)Haaro, Saksman, and Tamminen}]{Haario2001}
Haaro, H., Saksman, E., and Tamminen, J. (2001), \enquote{An Adaptive
  Metropolis Algorithm,} \textit{Bernoulli}, 7, 223--242.

\bibitem[{He et~al.(2015)He, Zhang, Ren, and Sun}]{resnet}
He, K., Zhang, X., Ren, S., and Sun, J. (2015), \enquote{Deep Residual Learning
  for Image Recognition,} \textit{CVPR}.

\bibitem[{Kendall and Gal(2017)}]{KendallGal2017}
Kendall, A. and Gal, Y. (2017), \enquote{What uncertainties do we need in
  Bayesian deep learning for computer vision,} in \textit{The 31st Conference
  on Neural Information Processing Systems (NIPS 2017)}, Long Beach, CA, USA.

\bibitem[{Kingma and Ba(2014)}]{KingmaB2015}
Kingma, D. and Ba, J. (2014), \enquote{Adam: A Method for Stochastic
  Optimization,} \textit{International Conference on Learning Representations},
  1--13.

\bibitem[{Li et~al.(2016)Li, Chen, Carlson, and Carin}]{Li2016PreconditionedSG}
Li, C., Chen, C., Carlson, D.~E., and Carin, L. (2016), \enquote{Preconditioned
  Stochastic Gradient {L}angevin Dynamics for Deep Neural Networks,} in
  \textit{AAAI}.

\bibitem[{Liang et~al.(2018)Liang, Li, and Zhou}]{Liang2018BNN}
Liang, F., Li, Q., and Zhou, L. (2018), \enquote{Bayesian Neural Networks for
  Selection of Drug Sensitive Genes,} \textit{\JASA}, 113, 955--972.

\bibitem[{Lin et~al.(2019)Lin, Khan, and Schmidt}]{Kahn2019}
Lin, W., Khan, M.~E., and Schmidt, M. (2019), \enquote{Fast and Simple
  Natural-Gradient Variational Inference with Mixture of Exponential-family
  Approximations,} in \textit{ICML}, {PMLR}, Proceedings of Machine Learning
  Research.

\bibitem[{Ma et~al.(2015)Ma, Chen, and Fox}]{Ma2015ACR}
Ma, Y.-A., Chen, T., and Fox, E.~B. (2015), \enquote{A Complete Recipe for
  Stochastic Gradient {MCMC},} in \textit{NIPS}.

\bibitem[{Mattingly et~al.(2002)Mattingly, Stuartb, and Highamc}]{mattingly02}
Mattingly, J., Stuartb, A., and Highamc, D. (2002), \enquote{Ergodicity for
  {SDE}s and {A}pproximations: {L}ocally {L}ipschitz {V}ector {F}ields and
  {D}egenerate {N}oise,} \textit{Stochastic Processes and their Applications},
  101, 185--232.

\bibitem[{Metropolis et~al.(1953)Metropolis, Rosenbluth, Rosenbluth, Teller,
  and Teller}]{Metropolis1953}
Metropolis, N., Rosenbluth, A., Rosenbluth, M., Teller, A., and Teller, E.
  (1953), \enquote{Equation of state calculations by fast computing machines,}
  \textit{Journal of Chemical Physics}, 21, 1087--1091.

\bibitem[{Nagapetyan et~al.(2017)Nagapetyan, Duncan, Hasenclever, Vollmer, L.,
  and Zygalakis}]{Nagapetyan2017}
Nagapetyan, T., Duncan, A., Hasenclever, L., Vollmer, S., L., S., and
  Zygalakis, K. (2017), \enquote{The True Cost of {SGLD},}
  \textit{ArXiv:1706.02692v1}.

\bibitem[{Nemeth and Fearnhead(2019)}]{NemethF2019}
Nemeth, C. and Fearnhead, P. (2019), \enquote{Stochastic Gradient {M}arkov
  Chain {M}onte {C}arlo,} \textit{arXiv:1907.06986}.

\bibitem[{Patterson and Teh(2013)}]{PattersonTeh2013}
Patterson, S. and Teh, Y.~W. (2013), \enquote{Stochastic Gradient Riemannian
  Langevin Dynamics on the Probability Simplex,} in \textit{Advances in Neural
  Information Processing Systems 26}, eds. Burges, C. J.~C., Bottou, L.,
  Welling, M., Ghahramani, Z., and Weinberger, K.~Q., Curran Associates, Inc.,
  pp. 3102--3110.

\bibitem[{Polson and Rockova(2018)}]{polson2018posterior}
Polson, N. and Rockova, V. (2018), \enquote{Posterior Concentration for Sparse
  Deep Learning,} \textit{arXiv preprint arXiv:1803.09138}.

\bibitem[{Qian(1999)}]{Qian1999}
Qian, N. (1999), \enquote{On the momentum term in gradient descent learning
  algorithms,} \textit{Neural Networks}, 12, 145--151.

\bibitem[{Raginsky et~al.(2017)Raginsky, Rakhlin, and Telgarsky}]{Maxim17}
Raginsky, M., Rakhlin, A., and Telgarsky, M. (2017), \enquote{Non-convex
  {L}earning via {S}tochastic {G}radient {L}angevin {D}ynamics: a nonasymptotic
  analysis,} \textit{Proceedings of Machine Learning Research}, 65, 1--30.

\bibitem[{Ruder(2016)}]{Ruder2016}
Ruder, S. (2016), \enquote{An overview of gradient descent optimization
  algorithms,} \textit{CoRR}, abs/1609.04747.

\bibitem[{Sato and Nakagawa(2014)}]{Sato2014ApproximationAO}
Sato, I. and Nakagawa, H. (2014), \enquote{Approximation Analysis of Stochastic
  Gradient {L}angevin Dynamics by using {F}okker-{P}lanck Equation and {I}to
  Process,} in \textit{ICML}.

\bibitem[{Song et~al.(2020)Song, Sun, Ye, and Liang}]{SongLiang2020}
Song, Q., Sun, Y., Ye, M., and Liang, F. (2020), \enquote{Extended Stochastic
  Gradient MCMC for Large-Scale Bayesian Variable Selection,} \textit{arXiv:},
  2002.02919v1.

\bibitem[{Staib et~al.(2019)Staib, Reddi, Kale, Kumar, and Sra}]{StaibRK2019}
Staib, M., Reddi, S., Kale, S., Kumar, S., and Sra, S. (2019),
  \enquote{Escaping saddle points with adaptive gradient methods,} in
  \textit{ICML}.

\bibitem[{Sutton(1986)}]{Sutton1986}
Sutton, R.~S. (1986), \enquote{Two Problems with Backpropagation and Other
  Steepest-Descent Learning Procedures for Networks,} in \textit{{P}roceedings
  of the Eighth Annual Conference of the Cognitive Science Society}, Hillsdale,
  NJ: Erlbaum.

\bibitem[{Teh et~al.(2016)Teh, Thiery, and Vollmer}]{teh2016consistency}
Teh, Y.~W., Thiery, A.~H., and Vollmer, S.~J. (2016), \enquote{Consistency and
  fluctuations for stochastic gradient {L}angevin dynamics,} \textit{The
  Journal of Machine Learning Research}, 17, 193--225.

\bibitem[{Tieleman and Hinton(2012)}]{TielemanH2012}
Tieleman, T. and Hinton, G. (2012), \enquote{Lecture 6.5-RMSProp: Divide the
  gradient by a running average of its recent magnitude,} \textit{COURSERA:
  Neural Networks for Machine Learning}, 4, 26--31.

\bibitem[{Vollmer et~al.(2016)Vollmer, Zygalakis, and Teh}]{VollmerZW2016}
Vollmer, S.~J., Zygalakis, K.~C., and Teh, Y.~W. (2016), \enquote{Exploration
  of the (Non-)Asymptotic Bias and Variance of Stochastic Gradient Langevin
  Dynamics,} \textit{Journal of Machine Learning Research}, 17, 1--48.

\bibitem[{Welling and Teh(2011)}]{Welling2011BayesianLV}
Welling, M. and Teh, Y.~W. (2011), \enquote{{B}ayesian Learning via Stochastic
  Gradient {L}angevin Dynamics,} in \textit{ICML}.

\bibitem[{Xu et~al.(2018)Xu, Chen, Zou, and Gu}]{Xu18}
Xu, P., Chen, J., Zou, D., and Gu, Q. (2018), \enquote{Global Convergence of
  {L}angevin Dynamics Based Algorithms for Nonconvex Optimization,} in
  \textit{NeurIPS}.

\bibitem[{Zeiler(2012)}]{AdaDelta2012}
Zeiler, M.~D. (2012), \enquote{{ADADELTA:} An Adaptive Learning Rate Method,}
  \textit{CoRR}, abs/1212.5701.

\bibitem[{Zhang et~al.(2019)Zhang, Li, Zhang, Chen, and Wilson}]{csgmcmc}
Zhang, R., Li, C., Zhang, J., Chen, C., and Wilson, A.~G. (2019),
  \enquote{Cyclical Stochastic Gradient MCMC for Bayesian Deep Learning,}
  \textit{arXiv preprint arXiv:1902.03932}.

\end{thebibliography}

\end{document}